\documentclass{article}

\usepackage{arxiv}

\usepackage{rotating} 
\usepackage{longtable} 
\usepackage{wrapfig} 
\usepackage{booktabs} 
\usepackage{float} 
\usepackage{natbib} 
\usepackage{gensymb} 
\usepackage{graphicx} 
\usepackage[linesnumbered,ruled,vlined]{algorithm2e} 
\usepackage[load-configurations=version-1]{siunitx} 
\usepackage{amsmath,amssymb,amsthm} 
\usepackage{hyperref} 

\newtheorem{definition}{Definition}
\newtheorem{lemma}{Lemma}
\newtheorem{proposition}{Proposition}
\newtheorem{theorem}{Theorem}

\begin{document}

\title{CantorNet: A Sandbox for Testing Geometrical and Topological Complexity Measures}

 \author{
Michal Lewandowski \\
Software Competence Center Hagenberg (SCCH) \\
\texttt{michal.lewandowski@scch.at} \\
\And
Hamid Eghbalzadeh \\
AI at Meta \\
\texttt{heghbalz@meta.com} \\
\And
Bernhard A. Moser \\
Software Competence Center Hagenberg (SCCH), \\
Johannes Kepler University of Linz (JKU) \\
\texttt{bernhard.moser@scch.at} \\
}

\maketitle

\begin{abstract}
Many natural phenomena are characterized by self-similarity, for example the symmetry of human faces, or a repetitive motif of a song. Studying of such symmetries will allow us to gain deeper insights into the underlying mechanisms of complex systems. Recognizing the importance of understanding these patterns, we propose a geometrically inspired framework to study such phenomena in artificial neural networks. To this end, we introduce \emph{CantorNet}, inspired by the triadic construction of the Cantor set, which was introduced by Georg Cantor in the $19^\text{th}$ century. In mathematics, the Cantor set is a set of points lying on a single line that is self-similar and  has a counter intuitive property of being an uncountably infinite null set. Similarly, we introduce CantorNet as a sandbox for studying self-similarity by means of novel topological and geometrical complexity measures. CantorNet constitutes a family of ReLU neural networks that spans the whole spectrum of possible Kolmogorov complexities, including the two opposite descriptions  (linear and exponential as measured by the description length). CantorNet's decision boundaries can be arbitrarily ragged, yet are analytically known. Besides serving as a testing ground for complexity measures, our work may serve to illustrate potential pitfalls in geometry-ignorant data augmentation techniques and adversarial attacks.\footnote{Accepted at NeurIPS Workshop Symmetry and Geometry  in Neural Representations 2024.}
\end{abstract}

\section{Introduction}
\label{sec:intro}

Neural networks perform extremely well in various domains, for example computer vision~\citep{Krizhevsky2012ImageNetCompetition} or speech recognition~\citep{Maas13rectifiernonlinearities}. 
Yet, this performance is not sufficiently understood from the mathematical perspective, and current advancements are not backed up by a formal mathematical analysis. 
We start by identifying  a lack of \emph{tractable examples} that allow us to study neural networks through the lens of  self-similarity of  objects they describe. This difficulty arises from the inherent statistical nature of these networks and the significant computational effort required to accurately determine the underlying geometry of the decision manifold.  We note that the use of constructed examples helps to illustrate certain characteristic effects, such as the
concentration of measure effects in high dimensions to explain the vulnerability against adversarial examples~\citep{gilmer2018adversarial}. Such examples are typically designed to underscore either the capabilities or limitations of neural architectures in exhibiting certain phenomena.
However, there exists a risk of oversimplification that might lead to an underappreciation of the complexities and challenges of handling real-world, high-dimensional, and noisy data. Despite these limitations, toy examples are valuable as they can be constructed to emphasise some properties which remain elusive at a larger scale. 
Further examples include  the XOR~\citep{minsky1969perceptrons}  or CartPole problem~\citep{Sutton-Barto-2018} which, despite their simplicity, have provided a controllable evaluation framework in their respective fields. 
Our analysis is further motivated by the fact that many natural phenomena feature some self-similarity as understood by symmetry, e.g., images~\citep{Wang2020}, audio tracks~\citep{Foote1999} or videos~\citep{AlemanFlores2004}.

Our work is closely related to the concepts such as Cantor set, fractals and the Kolmogorov complexity. In the following, we provide a brief background about these concepts.\\
\textbf{Cantor set.} Cantor set, introduced in mathematics  by Georg Cantor~\citep{cantor1879ueber}, is a set of points lying on a single line segment that has a number of counter intuitive properties such as being self-similar, and uncountably infinite, yet of Lebesgue measure 0.
It is obtained by starting with a line segment, partitioned into three equal sub-segments and recursively deleting the middle one, repeating the process an infinite number of times. It is used to illustrate how complex structures can arise from simple recursion rules~\citep{mandelbrot1983fractal}.
CantorNet is inspired by the  construction procedure of the Cantor set, inhering its fractal properties.\\ 
\textbf{Fractals.} In nature, some   complex shapes can be described in a  compact way. Fractals are a good example as they are self-similar geometric shapes whose intricate structure can be compactly encoded by a recursive formula (e.g.,~\cite{Koch1904kochcurve}). A number of measures have been developed to quantify the complexity of the fractals, e.g.,~\cite{mandelbrot1983fractal,Mandelbrot1995,ZMESKAL2013135}. 
On one hand, fractals are self-repetitive structures, what results in a compact description. On the other hand, some fractals can be represented as unions of polyhedral bodies, a less compact description. The construction of the CantorNet is inspired by the triadic representation of the Cantor set, and its opposite representations (complexity-wise) are based on recursed generating function and  union of polyhedral bodies.\\
\textbf{The Kolmogorov Complexity.} The Kolmogorov complexity~\citep{kolmogorov1965} quantifies the information conveyed by one object about another and can be applied to models by evaluating the shortest program that can reproduce a given output, as proposed as early as \cite{solomonoff1964}. In the context of neural networks, \cite{schmidhuber1997discovering} argues that prioritizing solutions with low Kolmogorov complexity enhances generalization. Although computing the exact Kolmogorov complexity of real-world architectures is unfeasible, approximations to the minimal description length~\citep{Grunwald2005} can be made by analyzing the number of layers and neurons, under the condition that the neural networks represent \emph{exactly} the same decision boundaries, which is also the case in our CantorNet analysis.

In summary, in this work we propose \emph{CantorNet}, an arbitrarily ragged decision surface, a natural candidate for testing various geometrical and topological measures. Furthermore, we study the Kolmogorov complexity of its two equivalent  constructions with ReLU  nets.
The rest of the paper is organized as follows. Section~\ref{sec:preliminaries} recalls some basic facts and fixes notation for the rest of the paper, and in Section~\ref{sec:CantorNet}, we describe different CantorNet constructions and representations. Finally, in Section~\ref{s:Conclusions}, we provide concluding remarks and possible future directions for our work.\footnote{Code available at~\url{https://github.com/michalmariuszlewandowski/CantorNet/}}

\section{Preliminaries}\label{sec:preliminaries}

We define a \emph{ReLU neural network} $\mathcal{N}:\mathcal{X}\rightarrow \mathcal{Y}$ with the total number of $N$ neurons as an alternating composition of   the ReLU function   $\sigma(x) := \max(x, 0)$   applied element-wise on the input $x$, and  affine functions with weights $W_k$ and biases $b_k$ at layer $k$. 
An input $x\in\mathcal{X}$ propagated through $\mathcal{N}$  generates non-negative activation values on each neuron.
A \textit{binarization} is  a mapping $\pi:\mathbb{R}^N \to \{0,1\}^N$ applied to a vector (here a concatenation of all hidden layer) $v=(v_1,\ldots,v_N)\in\mathbb{R}^N$ resulting in a binary vector $\{0,1\}^N$ by clipping strictly positive entries of $v$  to 1, and non-positive entries to 0, that is $\pi(v_i)=1$ if $v_i>0$, and $\pi(v_i)=0$ otherwise. An \emph{activation pattern} is the concatenation of all neurons after their binarization for a given input $\mathbf{x}$,  and represents an element in a binary hypercube $\mathcal{H}_N:=\{0,1\}^N$ where the dimensionality is equal to the number of hidden neurons in network $\mathcal{N}$. A \emph{linear region} is an element of a disjoint collection of subsets covering the input domain where the network behaves as an affine function~\citep{montufar2014number}. There is an one-to-one correspondence between an activation pattern and a linear region~\citep{shepeleva2020relu}.

\section{CantorNet}
\label{sec:CantorNet}

In this section, we define CantorNet as a ReLU neural network through repeating application of  weight matrices, similar to the fractal constructions. 
We then introduce an equivalent description through unionizing polyhedral bodies, which is less concise. 
We start the construction with two reshaped ReLU functions (Fig.~\ref{fig:cantor_set_cantornet}, left), and  modify them  to obtain a connected decision manifold with Betti numbers $b_i=0$ for $i\in\{0,1,2\}$, used to  characterizes the topological complexity, providing measures of connectivity, loops, and voids within the decision boundaries~\citep{BianchiniS14}. 
We consider  the function 
\begin{equation}
\label{eq:AA}
A:[0,1]\to[0,1] : x\mapsto \max\{-3x+1,0, 3x-2\},
\end{equation}
as the {\it generating function} and recursively nest it as 
\begin{equation}
\label{eq:nestedA}
A^{(k+1)}(x) := A(A^{(k)}(x)),\, A^{(1)}(x):= A(x).
\end{equation}
Based on the generating function, we can define the decision manifold  $R_k$ as:
\begin{equation}
\label{eq:Rk} 
R_k:= \{(x,y)\in [0,1]^2 : y \leq (A^{(k)}(x)+1)/2\}.
\end{equation}
For a better understanding of the decision manifolds, we have visualized  $R_1$, $R_2$, and $R_3$ in Fig.~\ref{fig:cantor_set_cantornet}.
The \textit{CantorNet} is given by the nested function defined by Eq.~\eqref{eq:nestedA}, which can be mapped to a ReLU neural network representation.
We define the CantorNet family as follows.

\begin{definition}
A ReLU net $\mathcal{N}$ belongs to the CantorNet family if
$\mathcal{N}^{(-1)}(0)= R_k$.   
\end{definition}

We further name the regions ``below'' and  ``above'' $R_k$ as the inset and the outset of CantorNet, respectively, as formalized in the Def.~\ref{def:classes_cantornet}.

\begin{definition}
\label{def:classes_cantornet}
We say that the manifold  $R_k$ given by Eq.~\eqref{eq:Rk} represents the inset of CantorNet, while its complement on the unit square represents the outset  (grey and white areas in Fig.~\ref{fig:TessellationModelA}, respectively).
\end{definition}

\subsection{Recursion-Based Construction}
\label{ss:ModelA}
Note that the  decision surface of $R_k$ (Eq.~\eqref{eq:Rk}) equals to the 0-preimage of a 
ReLU net $\mathcal{N}_A^{(k)}: [0,1]^2 \rightarrow \mathbb{R}$  with weights and biases defined as
\begin{equation}
\label{eq:ReprA}
W_{1}=
\begin{pmatrix}
-3 & 0\\
3 & 0\\
0 & 1
\end{pmatrix},
b_{1}=\begin{pmatrix}
1\\
-2\\
0\\
\end{pmatrix},
W_{2}=
\begin{pmatrix}
    1 & 1 & 0\\ 
0 & 0 & 1
\end{pmatrix}
\end{equation}
and 
the final layer  
$W_{L}=
\begin{pmatrix}
    -\frac 12 & 1
\end{pmatrix},
b_{L}=
\begin{pmatrix}
    - \frac 12
\end{pmatrix}.
$
For recursion depth $k$, we define 
 $\mathcal{N}_A^{(k)}$ as
\begin{equation}
\label{eq:A}
 \mathcal{N}_A^{(k)}(\mathbf{x}):=W_L\circ\sigma \circ g^{(k)}(\mathbf{x}) + b_L,
\end{equation}
where 
$
g^{(k+1)}(\mathbf{x}) := g^{(1)}(g^{(k)}(\mathbf{x})), \sigma$ is the ReLU function, and 
\begin{equation}
\label{eq:g1}
g^{(1)}(\mathbf{x}):= \sigma \circ W_2\circ \sigma\circ (W_1 \mathbf{x}^T+b_1).
\end{equation}
We use $\circ$ to denote the standard composition of functions. 
Fig.~\ref{fig:TessellationModelA} shows the linear regions resulting from the construction described in the Eq.~\eqref{eq:A} for recursion level $k=1$, as well as the linear regions with the corresponding activation patterns  from non-redundant neurons (we skip neurons which do not change their state).

\begin{figure}[!t]
\centering
\includegraphics[width=.99\textwidth]{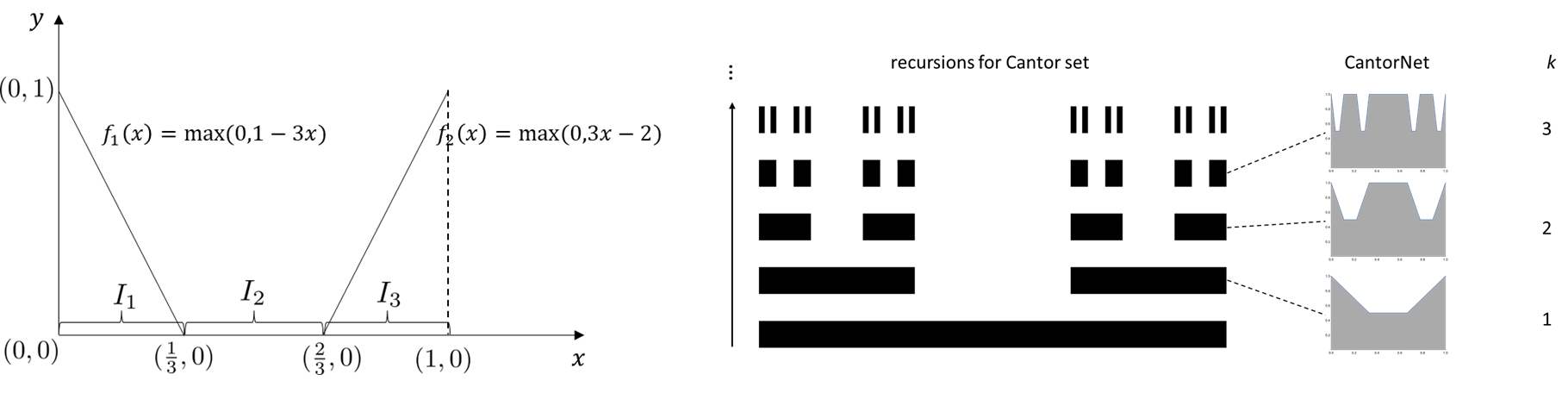}
\caption{Left: The first iteration of the 1-1 correspondence between the ReLU net $\widetilde{\mathcal{N}}_A$, induced by the generating function $A$, and the triadic number expansion shows the intervals $I_1, I_2, I_3$ correspond to the digits $\{0, 1, 2\}$, respectively. Right: CantorNet is inspired by the construction of the Cantor set~\citep{cantor1879ueber}.  
}
\label{fig:cantor_set_cantornet}
\end{figure}

\begin{figure*}[ht]
\centering \includegraphics[width=0.32\textwidth]{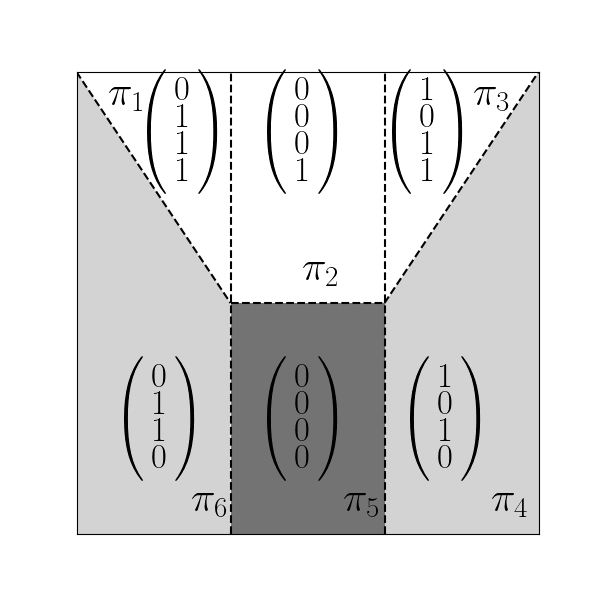}
\caption{Activation patterns $\pi_i$ induced by Eq.~\eqref{eq:A}. We skip  neurons with unchanged values.}
\label{fig:TessellationModelA}
\end{figure*}

\subsection{Triadic Expansion}
\label{sec:triadic_expansion}
In this section, we show that there exists an isomorphism between the triadic expansion, as described in Appendix~\ref{app:algorithm} in Alg.~\ref{alg:algebraic_triadic_expansion}, and the activation pattern $\pi_{\mathcal{N}_A}$ under $\mathcal{N}^{(k)}_A$. In the Triadic Expansion, we partition the interval $[0,1]$ into three intervals, $I_1=[0, \frac 13], I_2=(\frac13, \frac23), I_3=[\frac23, 1]$ (see Fig.~\ref{fig:cantor_set_cantornet}, left). Any  $x\in I_1\cup I_3$ can be described in a triadic system with an arbitrary precision $l$ as $x=\sum_{i=1}^l\frac{a_i}{3^i}$, where $a_i\in\{0,2\}$. Recall that the tessellation of the recursion-based model (Fig.~\ref{fig:TessellationModelA}) is obtained by  partitioning the rectangular domain $(I_1\cup I_3)\times[0,1]$ into increasingly fine rectangles through recursive applications of  $\mathbf{x}\mapsto g^{(k)}(\mathbf{x})$.
We  identify created linear regions by their activation patterns $\pi_{\mathcal{N}_A}$. Equivalently, we can  represent any  $x\in I_1\cup I_3$ using Alg.~\ref{alg:algebraic_triadic_expansion}, obtaining activation patterns as a sequence of ``0''s, and ``1''s. Each of these descriptions is unique, therefore there  exists an isomorphic relationship between the encoding  described in Alg.~\ref{alg:algebraic_triadic_expansion}, and the recursion-based description. 

\begin{lemma}[Computational Complexity of  Activation Patterns of $\mathcal{N}^{(k)}_A$]
\label{lemma:complexity_}
    Given an input $\mathbf{x}=(x_1,x_2)\in[0,1]^2$ and the recursion level $k$, its corresponding activation pattern $\pi(\mathbf{x})$ under the recursion-based representation $\mathcal{N}^{(k)}_A$ can 
    be computed in $O(k)$ operations. 
\end{lemma}

\begin{proof}
The complexity (as measured by the description length~\citep{Grunwald2005}) of the decision manifold given by Eq.~\eqref{eq:Rk} is equal to the complexity of its partition  into the linear regions defined in Sec.~\ref{sec:preliminaries}. 
To determine the minimal complexity of the partition it is necessary to solve  the following decision problem for $x_1$. Consider the partition into linear regions and its projection along the $y$-axis onto the $[0,1]\times \{0\}$.  Note that the resulting partition of $[0,1]$ is the same we obtain by constructing the Cantor set of level $k$, which corresponds to the triadic number expansion up to the $k^\text{th}$ digit.
The minimal complexity of solving this decision problem is therefore  $O(k)$.
\end{proof}

The proof of Lemma~\ref{lemma:complexity_} indicates the 1-1 correspondence between the triadic number expansions up to the $k^\text{th}$ digit and the activation pattern $\pi(x_1)$ of $x_1\in[0,1]$ under $\widetilde{\mathcal{N}}^{(k)}_A$, where $\widetilde{\mathcal{N}}^{(k)}_A$  represents the $1$-dim ReLU network up to the recursion level $k$, analogous to the construction given by Eq.~\eqref{eq:A}.
Observe that the outset (as in Def.~\ref{def:classes_cantornet}) intervals $I_1,I_2,I_3$ (as in  Fig.~\ref{fig:cantor_set_cantornet}, left) can be described with  activation patterns $\pi_{\mathcal{N}_A}(x)=[10111]$, for any $x\in I_1$, $\pi_{\mathcal{N}_A}(x)=[00101],$ for any $x\in I_2$ or $ \pi_{\mathcal{N}_A}(x)=[01111]$ for any $x\in I_3$ in the recursion-based representation (here we do not remove  neurons with constant values). 
Indeed, to obtain $\pi_{x}$ for $x\in I_1$ take any point $(x,y)\in I_1\times\{y\in [0,1]: y<f_1(x)\}$ ($f_1$ and $f_2$ as in Fig.~\ref{fig:cantor_set_cantornet}, left). After applying $\mathbf{x}\mapsto g^{(1)}(\mathbf{x})$, binarizing every neuron's value and concatenating them into a vector, we obtain a 5-dim vector (because $W_1\in\mathbb{R}^{\underline{3}\times 2}$, $W_2\in\mathbb{R}^{\underline{2}\times 3}$, and we  omit  $W_L, b_L$). In an analogous manner, we can obtain $\pi_{I_2}$ and $\pi_{I_3}$. Next, observe that each of the intervals $I_i$ can be further partitioned into $I_{i1}, I_{i2}, I_{i3}$, respectively for the left, center, and right segments. To describe these new segments, we increase the recursion level to $k=2$. It turns out that $I_{11}=[10111;10111]$,  a repetition of the pattern  $\pi_{I_1}$, and so forth for the remaining segments. This construction is iterated $k$ times, providing the sequence of subintervals 
\begin{equation}
\label{eq:subinterval}
    (I_{i_t})_{t=1}^{k}.
\end{equation} 

\subsection{Alternative Representation of CantorNet}
\label{ss:ModelB}
Observe that the pre-image of zero under a ReLU function (including shifting and scaling) is a closed set in $[0,1]^2$.
Since we consider a decision manifold $\mathcal{M}$ as a closed subset, which we referred to as  inset in Def.~\ref{def:classes_cantornet} in  Section~\ref{sec:CantorNet},
we use the closed pre-image of zero under the ReLU network $\mathcal{N}$ to model decision manifolds given by Eq.~\eqref{eq:Rk}. 
This means that the statement ``$\mathbf{x} \in \mathcal{M}$'' is true if $\mathcal{N}(\mathbf{x})=0$, and  the statement "$\mathbf{x} \in \mathcal{M}$" is false if 
$\mathcal{N}(\mathbf{x})>0$. 
This way the $\min$ operation refers to the union of sets, i.e., logical disjunction ``OR'', while the $\max$ operation refers to the  intersection of two sets, i.e., logical conjunction ``AND''. 
Note that the laws of Boolean logics also translate to this interpretation~\citep{klir1995fuzzy}. 
Further, note that any (non-convex) polytope is a geometric body that can be represented as the union of intersections of convex  polytopes.
A convex polytope can also  be represented as intersection of half-spaces, like the pre-image of zero under the ReLU function, i.e., a single-layered ReLU network. Thus, a decision manifold $\mathcal{M}$ given by a (non-convex) polytope can be represented by the minimum of maxima of single layered ReLU networks. 
Since the minimum operation can also be represented in terms of the $\max$ function, we obtain a ReLU representation for which it is justified to 
call it a \textit{disjunctive normal form} (DNF), as outlined by~\citeauthor{Moser22tessellationfiltering}.
For the simplest case, consider the function $h_1(x,y)$ (Fig.~\ref{fig:preimage}, left) that splits the unit square $[0,1]^2$ into parts where it takes positive and negative values, denoted  with $(1)$ and $(0)$, respectively. Observe that $(x,y)\in (\pi\circ h_1)^{(-1)}(0)$ if and only if $\max\{h_1(x,y),0\}=0$ (where $\pi$ is the binarization operator  described in Sec.~\ref{sec:preliminaries}). Similarly, $(x,y)\in (\pi\circ h_1)^{(-1)}(1)$ if and only if $\max\{h_1(x,y),0\}=h_1(x,y)$. In the case of two hyperplanes (Fig.~\ref{fig:preimage}, right), the polytope denoted with $(0,0)$  can be represented as $(x,y)\in (\pi\circ h_1)^{(-1)}(0)\cap (\pi\circ h_2)^{(-1)}(0)$ if and only if $ \max\{h_1(x,y), h_2(x,y),0\}=0$, and similarly for  the remaining  polytopes. 

Fig.~\ref{fig:1} represents the partition of the decision manifold given by Eq.~(\ref{eq:Rk}) into convex polytopes for $k=2$ and $k=3$, respectively. We utilize the minimum function to form their union, obtaining a ReLU network $\mathcal{N}_B^{(k)}$ that yields the same decision manifold as the recursion-based $\mathcal{N}_A^{(k)}$.

\begin{figure}
\centering
\includegraphics[width=0.55\textwidth]{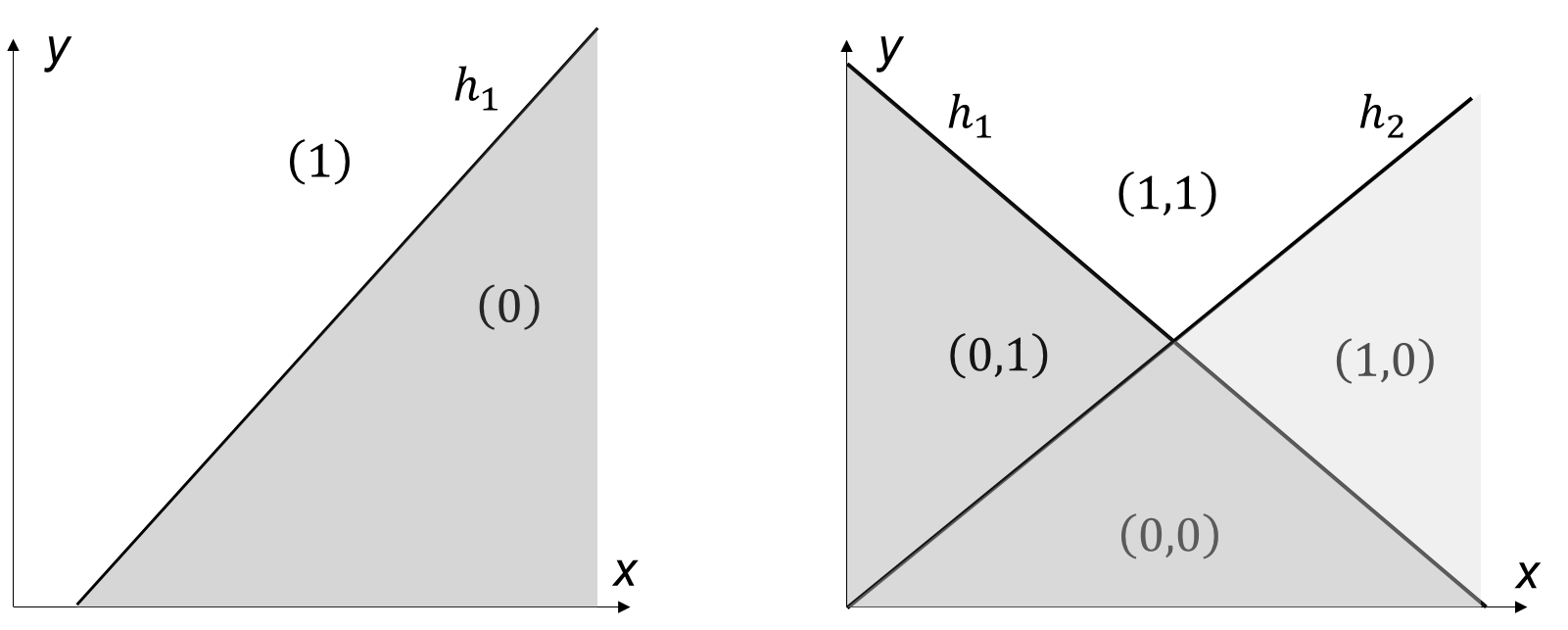}
\caption{The greyed regions  can be represented as a 0-preimage of $\pi\circ h_1$ (left), and the union of the 0-preimages of $\pi\circ h_1$ and $\pi\circ h_2$ (right).}
\label{fig:preimage}
\end{figure}

\begin{proposition}
\label{prop:dnf}
At the recursion level $k$, the decision boundary  given by Eq.~\eqref{eq:Rk} can be constructed as a disjunctive normal form 
\begin{equation}
\label{eq:dnf_formula}
    \{x,y\in[0,1]^2:\min(h_1(x,y),h_2(x,y),h_{r(k)}(x,y), D_1, \ldots, D_{\lfloor r(k)/4\rfloor+1},0)=0\},
\end{equation}
where $h_i:\mathbb{R}^2\to\mathbb{R}$ are  affine functions indexed with $i=1,\ldots,r(k)$. The labeling function $r(k):\mathbb{N}\to\mathbb{N}$ is given as $r(k)=2^{k+1}-1$ (see Fig.~\ref{fig:1}), and $D:\mathbb{R}^2\to\mathbb{N}$ denotes a ``dent'' given by 
\begin{equation}
\label{eq:dent}
D_l:=\max(h_{4l-1}, h_{4l}, h_{4l+1}).
\end{equation}
\end{proposition}
To provide a better overview for the reader, in Table~\ref{tab:min_max_shape_repr_B} we list the constructions for the different recursion levels.
\begin{figure}[ht]
\centering
\includegraphics[width=0.67\textwidth]{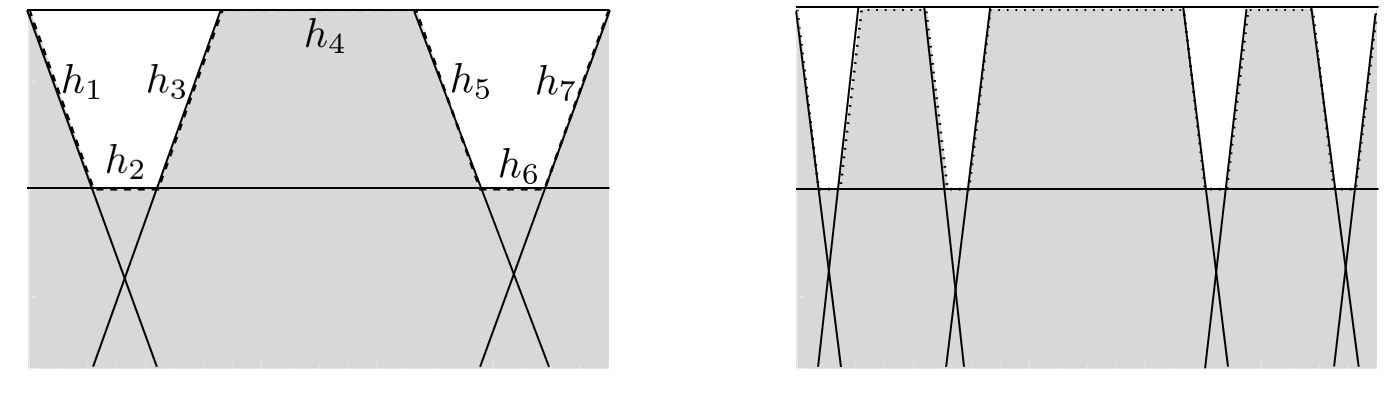}
\caption{Decision surfaces ($k=2,3$) of~\eqref{eq:Rk} with labeled functions $h_i$.}
\label{fig:1}
\end{figure}
\begin{table}[h]
    \centering
    \caption{Min/max  shape description for $k$ recursions.}
    \begin{tabular}{c|c}
      recursion & formula of decision manifold $R_k$ \\ \hline
        $1$ &  $\{x,y:\min(h_1,h_2,h_3)=0\}$\\
        $2$ & $\{x,y:\min(\underbrace{h_1,h_2,h_7,}_{} \underbrace{\max(h_3,h_4,h_5)}_{})=0\}$\\
        $\ldots$ & $\ldots$ \\
        $k$  & $\{x,y:\min(\underbrace{h_1,h_2,h_{r(k)},}_{\text{external half-spaces}} \underbrace{D_1, \ldots, D_{\lfloor r(k)/4\rfloor+1}}_{\text{``dents''~\eqref{eq:dent}}})=0\}$
    \end{tabular}
    \label{tab:min_max_shape_repr_B}
\end{table}
We sketch an inductive proof of the Proposition~\ref{prop:dnf} in Appendix~\ref{app:proof_min_max}. Further, note that the min function can be expressed as a ReLU network (Appendix~\ref{app:min_as_relu}).

\section{Complexity of Neural Representations}

The complexity of an object can be measured in many ways, for example its  description's length~\citep{kolmogorov1965}. 
Preference for a more concise description can be argued from multiple angles, for example using the principle of the Occam's razor or lower Kolmogorov complexity. The latter is typically non-computable, necessitating reliance on its approximations. 
However,   in case of models with consistent decision boundaries (e.g., neural networks), their size, both in terms of the number of layers and the number of neurons,  can be used as  approximation for complexity~\citep{Grunwald2005}.

\begin{lemma}
\label{lemma:neurons}
    At the recursion  level $k$, the number of neurons   of the recursion-based representation is $O(k)$, while for  the disjunctive normal form representation it  is $O(2^{k})$. 
\end{lemma}
\begin{proof}
For the recursion-based representation the result is straightforward.
The DNF construction  relies on the application of $\mathbf{A}$ and $\mathbf{S}$ (see App.~\ref{app:proof_min_max}). For  recursion level $k$, we have (recall Eq.~\eqref{eq:dnf_formula})
$3+3(\lfloor r(k)/4\rfloor +1)=3\lfloor r(k)/4\rfloor+6=:z(k)$
rows of  $\mathbf{A}$, equal to the number of neurons.   By applying $\mathbf{A}$ and $\mathbf{S}$ at least $\lceil \log_2 z(k) \rceil$ times~\citep{arora16understandingReLU}, we arrive at \begin{equation*}
\label{eq:dnf_kolmogorov_complexity}
    \left(\frac 34 \left(2^{k+1}-1\right)+6 \right)\sum_{i=0}^{\lceil \log_2 z(k)\rceil}\frac{1}{2^i}\leq \frac 32 \left(2^{k+1}-1\right) = O(2^{k}).
\end{equation*}
\end{proof}
It requires an algorithm of Kolmogorov complexity of order $k$ to enumerate all numbers in $[0,1]$ with triadic number expansion up to $k$ digits. Since the  recursion-based ReLU net $\mathcal{N}_A$ is constructed by a repetitive application  of two layers with constant number of neurons (namely five) and the same weights,  its Kolmogorov complexity after $k$ iterations  is of order $k$.  As a recursion step given by Eq.~\eqref{eq:A} and Eq.~\eqref{eq:g1} is equivalent to a recursion in the triadic number expansion (Alg.~\ref{alg:algebraic_triadic_expansion}), which is of the minimal order of Kolmogorov complexity, there cannot exist an equivalent ReLU network of strictly lower order of Kolmogorov complexity.
This means that $\mathcal{N}_A$ is of minimal description length in terms of order of the number of neurons $N=N(k)$, thus 
\begin{equation}
    \label{eq:complextiy_cantornet}O(N(k))=O(k). 
\end{equation}
\begin{theorem} 
   The  recursion-based ReLU representation given by the Eq.~\eqref{eq:ReprA} is of minimal complexity order in terms of the number of neurons (in the sense of Eq.~\eqref{eq:complextiy_cantornet}).
\end{theorem}

This way, we obtain an example of a ReLU network of proven minimal description length in terms of its number of neurons,  a property  hardly provable  by  statistical means.
Observe that the above does not hold for singular numbers from $[0,1]$: if we consider $x=\frac 16=0...0...02...20...02...2...0...0$,  with $n=k^2$ digits after the ternary point arranged in $k$ alternating blocks of zeros and twos, then it has Kolmogorov complexity $O(k) = O(\sqrt{n})$.
Note that both representations have the same order of the number of layers.

\begin{lemma}
\label{lemma:layers}
    At the recursion level $k$ both described representations of CantorNet have $O(k)$ layers.
\end{lemma}
Note that both constructions are equivalent as understood by the equality of their preimages. Though simple by construction, the family of CantorNet  ReLU networks is rich in terms of representation variants, ranging from a minimal (linear in $k$) complex solution  to an exponentially  complex one. An intermediate example  would be starting with the recursion based representation for a number of layers, and then concatenating corresponding disjunctive normal form representation.

\section{Conclusions and Discussion}
\label{s:Conclusions}

In this paper we have proposed CantorNet, a family of ReLU neural networks inspired by fractal geometry that can be tuned arbitrarily close to a fractal.
The resulting geometry of CantorNet's decision manifold, the induced tessellation and activation patterns can be derived in two ways. 
This  makes it a natural candidate for studying concepts related to the activation space. 
Note that, although CantorNet is a hand-designed example, it is not an abstract invention -  real world data, such as images, music, videos also display fractal nature, as understood by self-similarity. We believe that our work, although seemingly remote from the current mainstream of the machine learning research, will provide the community with a set of examples
to study ReLU neural networks as mappings between the Euclidean input space and the space of activation patterns, currently under investigation.

\section*{Acknowledgments}
SCCH's research was carried out under the Austrian COMET program (project S3AI with FFG no. 872172), which is funded by the Austrian ministries BMK, BMDW, and the province of Upper Austria.


\newpage
\appendix

\section{Activation Code by Triadic Expansion\label{app:algorithm}}

\begin{algorithm}[htbp]
\caption{Activation Code of $\widetilde{\mathcal{N}}^{(k)}_A$ by Triadic Expansion}
\label{alg:algebraic_triadic_expansion}
\KwIn{$x_1 \in [0,1]$, recursion level $k$}
\KwOut{The activation code $\pi:=\pi_{\widetilde{\mathcal{N}}_A}(x_1)$}
Define $f_1(x) := 1 - 3x, f_2(x) := 3x - 2$\;
Define intervals: $I_1:=[0,\frac{1}{3}], I_2:=(\frac{1}{3},\frac{2}{3}), I_3:=[\frac{2}{3},1]$\;
\textbf{Termination} $\gets$ \textbf{False}\;
$\pi$ = [] \tcp*[h]{A pattern holder}\;
$j \gets 0$ \tcp*[h]{Iteration counter}\;

\While{\textbf{not} Termination}{
    $j \gets j + 1$\;
    \If{$j = k$}{
        \textbf{Termination} $\gets$ \textbf{True}\;
    }
    Update interval $J_{i_j} := I_{i_1 \ldots i_j}$, $i_j \in \{1,2,3\}$ \tcp*[h]{see~\eqref{eq:subinterval}}\;
    \eIf{$x_1 \in J_1$}{
        $x_1 \gets f_1(x_1)$\;
        $\pi$.append(0)\;
    }{
        \eIf{$x_1 \in J_2$}{
            \textbf{Termination} $\gets$ \textbf{True} \tcp*[h]{Exit the loop if $x_1$ is in $I_2$}\;
        }{
            $x_1 \gets f_2(x_1)$\;
            $\pi$.append(1)\;
        }
    }
}
\end{algorithm}

\section{Min/max Shape Description of CantorNet}
\label{app:proof_min_max}

\begin{proof}
We   sketch  the inductive proof of Proposition~\ref{prop:dnf}.  For the base case $k=1$, the decision manifold $R_1$ is composed of the union of three half-spaces, expressed as $\{ x,y\in[0,1]^2 :\min(h_1,h_2,h_3)= 0\}$. For $k=2$ (shown on the left in Fig.~\ref{fig:1}), we see a dent in the middle of the figure. Points lying above this dent satisfy the condition $\{x,y\in[0,1]^2 :\max(h_3,h_4,h_5)> 0\}$. To reconstruct the complete decision manifold $R_2$, we unionize the outer half-spaces $h_1,h_2,h_7$ with the aforementioned dent. This results in the following set:
$$
\{x,y\in[0,1]^2:\min(h_1,h_2,h_7,\max(h_3,h_4,h_5))= 0\}.
$$
For the inductive step, let's consider an arbitrary recursion depth $k$. Suppose that the decision boundaries $R_k$ consist of points $x,y\in[0,1] $ such that
\begin{equation}
\label{eq:decision_manifold_repr_B}
\min\left(h_1,h_2,h_{r(k)}, D_1, \ldots, D_{\lfloor r(k)/4\rfloor+1}\right)=0,
\end{equation}
where $D_l$ is as follows
\begin{equation}
D_l:=\max(h_{4l-1}, h_{4l}, h_{4l+1})
\end{equation}
 describes dents.
Observing the pattern established by Eq.~\ref{eq:AA}, it becomes clear that incrementing the recursion depth from $k$ to $k+1$ doubles the number of  nested maximum functions. In other words, twice the previous number of ``dents'' emerge in our structure. This observation aligns with Eq.~\ref{eq:decision_manifold_repr_B}. This and the construction of half-spaces $h_i$   assure that all the half-spaces agree with the decision manifold $R_{k+1}$.
\end{proof}

\section{Min as a ReLU Net}
\label{app:min_as_relu}

\begin{theorem}
\label{thm:matrix_repr_of_relu}
The minimum function $\min:\mathbb{R}^d\to\mathbb{R}$ can be expressed as a ReLU neural network with weights $\{0,\pm1\}$.
\end{theorem}

\begin{proof}
We first show that in the base cases of even and odd number of variables, $d=2$ and $d=3$  respectively, we can recover the minimum element by a hand-designed  ReLU neural architecture. Recall that $\sigma(x):=\max(x,0)$, applied element-wise.
For $d=2$ and elements $x_1,x_2\in\mathbb{R}$ it holds that 
\begin{align}
    \min(x_1,x_2)&=x_2 + \min(x_1-x_2,0) \nonumber \\
    &=x_2-\max(x_2-x_1,0) \nonumber \\ 
    &=\sigma(x_2)-\sigma(-x_2)-\sigma(-x_1+x_2), \label{eq:derivation_repr_B}
\end{align}
which can be represented as a  ReLU net
\begin{equation*}
\min(x_1,x_2)=\underbrace{\begin{pmatrix}
1 & -1 & -1 
\end{pmatrix}}_{\mathbf{S}}\sigma
\underbrace{\begin{pmatrix}
0 & 1 \\
0 & -1\\
-1 & 1
\end{pmatrix}}_{\mathbf{A}}\begin{pmatrix}
x_1 \\
x_2
\end{pmatrix}.
\end{equation*}
For $d=3$, $\min(x_1,x_2,x_3)=\min(\min(x_1,x_2),x_3)$, which can be expressed by a ReLU neural network as follows:
\begin{equation*}
    \mathbf{S}\sigma\mathbf{A}\begin{pmatrix}
        1 & -1 & -1 & 0 & 0\\
        0&0&0&1&-1
    \end{pmatrix}
    \sigma
    \begin{pmatrix}
    0 & 1 & 0 \\
    0 & -1 & 0 \\
    -1 & 1 & 0 \\
    0 & 0 & 1 \\
    0 & 0 & -1 
    \end{pmatrix}
    \begin{pmatrix}
    x_1\\
    x_2\\
    x_3
    \end{pmatrix},
\end{equation*}
where we first recover the $\min(x_1,x_2)$ and leave $x_3$ unchanged, resulting in  $(\min(x_1,x_2), x_3)$, and then we recover the minimum element using the base case for $d=2$.

\textit{Inductive step.} We start with the inductive step for an even number of elements. 
Let $\mathbf{x}\in\mathbb{R}^{d},\ d=2l$ for $l\in\mathbb{N}$, and suppose that with $\mathbf{S}^\prime\sigma \left( \mathbf{A}^\prime \mathbf{x}\right),\mathbf{S}^\prime\in\{0,\pm1\}^{d\times 3d}, \mathbf{A}^\prime\in\{0,\pm1\}^{3d\times 2d}$  
\begin{equation*}
\mathbf{S}^\prime:=
\begin{pmatrix}
\mathbf{S} & \mathbf{0}_{1\times3} & \ldots &  \mathbf{0}_{1\times3}\\
\mathbf{0}_{1\times3} & \mathbf{S} & \mathbf{0}_{1\times3} & \ldots\\
\ldots & \ldots & \ldots & \ldots\\
\mathbf{0}_{1\times3} & \ldots & \mathbf{0}_{1\times3} & \mathbf{S}
\end{pmatrix},\end{equation*}
and 
\begin{equation*}
\mathbf{A}^\prime:=
\begin{pmatrix}
\mathbf{A} & \mathbf{0}_{3\times2} & \ldots &  \mathbf{0}_{3\times2}\\
\mathbf{0}_{3\times2} & \mathbf{A} & \mathbf{0}_{3\times2} & \ldots\\
\ldots & \ldots & \ldots & \ldots\\
\mathbf{0}_{3\times2} & \ldots & \mathbf{0}_{3\times2} & \mathbf{A}
\end{pmatrix},
\end{equation*}
we group elements in pairs, reducing the problem to $\min(\min(x_1,x_2),\ldots,\min(x_{d-1},x_{d}))$ ($\mathbf{0}_{m\times n}$ is a matrix of zeros with $m$ rows and $n$ columns).
Then, for  $\mathbf{x}\in\mathbb{R}^{d+2}$, we use the inductive step to recover $\min(\min(x_1,x_2),\ldots,\min(x_{d-1},x_{d}), \min(x_{d+1},x_{d+2}))$ by using $\mathbf{S}^{\prime\prime}\sigma  (\mathbf{A}^{\prime\prime} \mathbf{x})$ where
\begin{equation}
\label{eq:induction_step}
    \mathbf{S}^{\prime\prime}:=\begin{pmatrix}
    \mathbf{S}^\prime & \mathbf{0}_{d\times 3} \\
    \mathbf{0}_{1\times 3d} & \mathbf{S}
    \end{pmatrix}, \ 
    \mathbf{A}^{\prime\prime}:=\begin{pmatrix}
    \mathbf{A}^\prime & \mathbf{0}_{3d\times 2} \\
    \mathbf{0}_{3\times 2d} & \mathbf{A} 
    \end{pmatrix}.
\end{equation}
Indeed, if with $\mathbf{S}^\prime$ and $\mathbf{A}^\prime$ we group variables into pairs, then we can also group in pair two additional elements. 
Now let's consider the inductive step for an odd case $\mathbf{x}\in\mathbb{R}^{d+1}$.   We can recover $\min(\min(x_1,x_2),\ldots,\min(x_{d-1},x_{d}), x_{d+1})$ by using $\mathbf{S}^{\bullet}\sigma(\mathbf{A}^{\bullet}\mathbf{x})$, where 
\begin{equation*}
    \mathbf{S}^{\bullet}:=\begin{pmatrix}
    \mathbf{S}^\prime & \mathbf{0}_{d\times 1} & \mathbf{0}_{d\times 1} \\
    \mathbf{0}_{2\times 3d} & 1 & -1
    \end{pmatrix}, \ 
    \mathbf{A}^{\bullet}:=\begin{pmatrix}
    \mathbf{A}^\prime & \mathbf{0}_{3d\times 1} \\
    \mathbf{0}_{1\times 2d} & 1 \\
    \mathbf{0}_{1\times 2d} & -1
    \end{pmatrix}.
\end{equation*}
For $\mathbf{x}\in\mathbb{R}^{d+3}$, we extend $\mathbf{S}^{\bullet}$ and $\mathbf{A}^{\bullet}$ as in Eq.~\eqref{eq:induction_step}, pairing the last two elements. 
Recursively applying $\mathbf{S}^*\sigma (\mathbf{A}^* \mathbf{x})$ (where $^*$ means that the dimensionality must be chosen appropriately) groups the elements in pairs and eventually returns the minimum element. 
\end{proof}



\end{document}